\documentclass{article}
\usepackage{times}
\usepackage{graphicx}
\usepackage{subfigure} 
\usepackage{natbib}

\usepackage[accepted]{icml2017} 
\usepackage{epstopdf}
\usepackage[normalem]{ulem}
\usepackage{hyperref}
\usepackage{latexsym}
\usepackage{amsmath}
\usepackage{amsfonts}
\usepackage{amsthm}
\usepackage{rotating}
\usepackage{multirow}
\usepackage{bm}
\usepackage{longtable}
\usepackage{colortbl}
\usepackage{amssymb}
\usepackage{algorithm}
\usepackage{algorithmic}
\usepackage[utf8]{inputenc} 
\usepackage[T1]{fontenc} % use 8-bit T1 fonts

\newcommand{\paramAB}{\theta_{xy}}%{\overrightarrow{\Theta}}
\newcommand{\paramBA}{\theta_{yx}} %{\overleftarrow{\Theta}}

\newtheorem{theorem}{Theorem}

\newtheorem{definition}{Definition}

\icmltitlerunning{Dual Supervised Learning}

\begin{document} 
\twocolumn[
\icmltitle{Dual Supervised Learning}



% you can specify symbols, otherwise they are numbered in order
\icmlsetsymbol{equal}{*}

\begin{icmlauthorlist}
	\icmlauthor{Yingce Xia}{ustc}
	\icmlauthor{Tao Qin}{msra}
	\icmlauthor{Wei Chen}{msra}
	\icmlauthor{Jiang Bian}{msra}
    \icmlauthor{Nenghai Yu}{ustc}
	\icmlauthor{Tie-Yan Liu}{msra}
\end{icmlauthorlist}

\icmlaffiliation{ustc}{School of Information Science and Technology, University of Science and Technology of China, Hefei, Anhui, China}
\icmlaffiliation{msra}{Microsoft Research, Beijing, China}
%\icmlaffiliation{ed}{University of Edenborrow, Edenborrow, United Kingdom}

\icmlcorrespondingauthor{Tao Qin}{taoqin@microsoft.com}
%\icmlcorrespondingauthor{Eee Pppp}{ep@eden.co.uk}

\icmlkeywords{dual learning, deep learning}

\vskip 0.3in
]


\printAffiliationsAndNotice{}  % leave blank if no need to mention equal contribution
%\footnotetext{hi}	

\begin{abstract}
Many supervised learning tasks are emerged in dual forms, e.g., English-to-French translation vs. French-to-English translation, speech recognition vs. text to speech, and image classification vs. image generation. Two dual tasks have intrinsic connections with each other due to the probabilistic correlation between their models. This connection is, however, not effectively utilized today, since people usually train the models of two dual tasks separately and independently. In this work, we propose training the models of two dual tasks simultaneously, and explicitly exploiting the probabilistic correlation between them to regularize the training process. For ease of reference, we call the proposed approach \emph{dual supervised learning}. We demonstrate that dual supervised learning can improve the practical performances of both tasks, for various applications including machine translation, image processing, and sentiment analysis.
\end{abstract}

\section{Introduction}\label{sec:intro}
Deep learning brings state-of-the-art results to many artificial intelligence tasks, such as neural machine translation \cite{wu2016google}, image classification \cite{he2015deep,he2016identity}, image generation \cite{van2016pixel,van2016conditional}, speech recognition \cite{graves2013speech,amodei2016deep}, and speech generation/synthesis \cite{oord2016wavenet}. 

Interestingly, we find that many of the aforementioned AI tasks are emerged in dual forms, i.e., the input and output of one task are exactly the output and input of the other task respectively. Examples include translation from language A to language B vs. translation from language B to A, image classification vs. image generation, and speech recognition vs. speech synthesis. Even more interestingly (and somehow surprisingly), this natural duality is largely ignored in the current practice of machine learning. That is, despite the fact that two tasks are dual to each other, people usually train them independently and separately. Then a question arises: Can we exploit the duality between two tasks, so as to achieve better performance for both of them? In this work, we give a positive answer to the question.

To exploit the duality, we formulate a new learning scheme, which involves two tasks: a primal task and its dual task. The primal task takes a sample from space $\mathcal{X}$ as input and maps to space $\mathcal{Y}$, and the dual task takes a sample from space $\mathcal{Y}$ as input and maps to space $\mathcal{X}$. Using the language of probability, the primal task learns a conditional distribution $P(y|x;\paramAB)$ parameterized by $\paramAB$, and the dual task learns a conditional distribution $P(x|y;\paramBA)$ parameterized by $\paramBA$, where $x\in\mathcal{X}$ and $y\in\mathcal{Y}$. In the new scheme, the two dual tasks are jointly learned and their structural relationship is exploited to improve the learning effectiveness. We name this new scheme as \emph{dual supervised learning} (briefly, DSL).

There could be many different ways of exploiting the duality in DSL. In this paper, we use it as a regularization term to govern the training process. Since the joint probability $P(x,y)$ can be computed in two equivalent ways: $P(x,y)=P(x)P(y|x)=P(y)P(x|y)$, for any $x\in \mathcal{X}, y\in \mathcal{Y}$, ideally the conditional distributions of the primal and dual tasks should satisfy the following equality:
\begin{equation}
P(x)P(y|x;\paramAB)=P(y){P}(x|y;\paramBA). %\forall x\in \mathcal{X}, y\in \mathcal{Y}.
\label{eq:high_level_idea}
\end{equation}
However, if the two models (conditional distributions) are learned separately by minimizing their own loss functions (as in the current practice of machine learning), there is no guarantee that the above equation will hold. The basic idea of DSL is to jointly learn the two models $\paramAB$ and $\paramBA$ by minimizing their loss functions subject to the constraint of Eqn.\eqref{eq:high_level_idea}. By doing so, the intrinsic probabilistic connection between $\paramBA$ and $\paramAB$ are explicitly strengthened, which is supposed to push the learning process towards the right direction. To solve the constrained optimization problem of DSL, we convert the constraint Eqn.\eqref{eq:high_level_idea} to a penalty term by using the method of Lagrange multipliers~\cite{boyd2004convex}. Note that the penalty term could also be seen as a data-dependent regularization term.

To demonstrate the effectiveness of DSL, we apply it to three artificial intelligence applications~\footnote{In our experiments, we chose the most cited models with either open-source codes or enough implementation details, to ensure that we can reproduce the results reported in previous papers. All of our experiments are done on a single Telsa K40m GPU.}:
%, as shown below. Experimental results show that DSL can effectively improve the dual tasks simultaneously.

(1) \emph{Neural Machine Translation (NMT)} We first apply DSL to NMT, which formulates machine translation as a sequence-to-sequence learning problem, with the sentences in the source language as inputs and those in the target language as outputs. The input space and output space of NMT are symmetric, and there is almost no information loss while mapping from $x$ to $y$ or from $y$ to $x$. Thus, symmetric tasks in NMT fits well into the scope of DSL. Experimental studies illustrate significant accuracy improvements by applying DSL to NMT: $+2.07/0.86$ points measured by BLEU scores for English$\leftrightarrow$French translation, $+1.37/0.12$ points for English$\leftrightarrow$Germen translation and $+0.74/1.69$ points on English$\leftrightarrow$Chinese.

(2) \emph{Image Processing} We then apply DSL to image processing, in which the primal task is image classification and the dual task is image generation conditioned on category labels. Both tasks are hot research topics in the deep learning community. We choose ResNet \cite{he2015deep} as our baseline for image classification, and PixelCNN++\cite{tim2017pixel} as our baseline for image generation. Experimental results show that on CIFAR-10, DSL could reduce the error rate of ResNet-110 from $6.43$ to $5.40$ and obtain a better image generation model with both clearer images and smaller \emph{bits per dimension}. Note that these primal and dual tasks do not yield a pair of completely symmetric input and output spaces since there is information loss while mapping from an image to its class label. Therefore, our experimental studies reveal that DSL can also work well for dual tasks with information loss. 

(3) \emph{Sentiment Analysis} Finally, we apply DSL to sentiment analysis, in which the primal task is sentiment classification (i.e., to predict the sentiment of a given sentence) and the dual one is sentence generation with given sentiment polarity. Experiments on the IMDB dataset show that DSL can improve the error rate of a widely-used sentiment classification model by $0.9$ point, and can generate sentences with clearer/richer styles of sentiment expression. 

All of above experiments on real artificial intelligence applications have demonstrated that DSL can improve practical performance of both tasks, simultaneously.

\section{Framework}\label{sec:model_theory}
In this section, we formulate the problem of {\em dual supervised learning} (DSL), describe an algorithm for DSL, and discuss its connections with existing learning schemes and its application scope.

\subsection{Problem Formulation}
To exploit the duality, we formulate a new learning scheme, which involves two tasks: a primal task that takes a sample from space $\mathcal{X}$ as input and maps to space $\mathcal{Y}$, and a dual task takes a sample from space $\mathcal{Y}$ as input and maps to space $\mathcal{X}$.

Assume we have $n$ training pairs $\{(x_i,y_i)\}_{i=1}^{n}$ {\em i.i.d.} sampled from the space $\mathcal{X}\times\mathcal{Y}$ according to some unknown distribution $P$. Our goal is to reveal the bi-directional relationship between the two inputs $x$ and $y$. To be specific, we perform the following two tasks: (1) the primal learning task aims at finding a function $f: \mathcal{X}\mapsto\mathcal{Y}$ such that the prediction of $f$ for $x$ is similar to its real counterpart $y$; (2) the dual learning task aims at finding a function $g: \mathcal{Y}\mapsto\mathcal{X}$ such that the prediction of $g$ for $y$ is similar to its real counterpart $x$. The dissimilarity is penalized by a loss function. Given any $(x,y)$, let $\ell_1(f(x),y)$  and $\ell_2(g(y),x)$ denote the loss functions for $f$ and $g$ respectively, both of which are mappings from $\mathcal{X}\times\mathcal{Y}$ to $\mathbb{R}$.

A common practice to design $(f,g) $ is the maximum likelihood estimation based on the parameterized conditional distributions $P(\cdot|x; \paramAB)$ and $P(\cdot|y; \paramBA)$: 
\begin{align*}
f(x;\paramAB)&\triangleq\arg\max_{y^\prime\in\mathcal{Y}}P(y^\prime|x; \paramAB),\\
g(y;\paramBA)&\triangleq\arg\max_{x^\prime\in\mathcal{X}}P(x^\prime|y; \paramBA),
\end{align*}
where $\paramAB$ and $\paramBA$ are the parameters to be learned. 

By standard supervised learning, the primal model $f$ is learned by minimizing the empirical risk in space $\mathcal{Y}$:

\qquad\qquad$\min_{\paramAB }(1/n)\textstyle{\sum_{i=1}^n}\ell_1(f(x_i;\paramAB), y_i);$

and dual model $g$ is learned by minimizing the empirical risk in space $\mathcal{X}$:

\qquad\qquad$\min_{\paramBA }(1/n)\textstyle{\sum_{i=1}^n}\ell_2(g(y_i;\paramBA),x_i).$

Given the duality of the primal and dual tasks, if the learned primal and dual models are perfect, we should have 
\begin{equation*}
P(x)P(y|x;\paramAB)=P(y)P(x|y;\paramBA)=P(x,y), \forall x, y.
\end{equation*}
We call this property \emph{probabilistic duality}, which serves as a necessary condition for the optimality of the learned two dual models. 

By the standard supervised learning scheme, probabilistic duality is not considered during the training, and the primal and the dual models are trained independently and separately. Thus, there is no guarantee that the learned dual models can satisfy probabilistic duality. To tackle this problem, we propose explicitly reinforcing the empirical probabilistic duality of the dual modes by solving the following multi-objective optimization problem instead:
\begin{equation}
\begin{aligned}
&\text{objective 1: } \min_{\paramAB}\;(1/n)\textstyle{\sum_{i=1}^n}\ell_1(f(x_i;\paramAB), y_i), \\
&\text{objective 2: } \min_{\paramBA}\;(1/n)\textstyle{\sum_{i=1}^n}\ell_2(g(y_i;\paramBA), x_i), \\
&\text{s.t. } P(x)P(y|x;\paramAB)=P(y)P(x|y;\paramBA), \forall x,y,
\end{aligned}
\label{eq:high_level_obj}
\end{equation}
where $P(x)$ and $P(y)$ are the marginal distributions. We call this new learning scheme \emph{dual supervised learning} (abbreviated as DSL). 

We provide a simple theoretical analysis which shows that DSL has theoretical guarantees in terms of generalization bound. Since the analysis is straightforward, we put it in Appendix~\ref{app:theory_analysis}.

\subsection{Algorithm Description}
In practical artificial intelligence applications, the ground-truth marginal distributions are usually not available. As an alternative, we use the empirical marginal distributions $\hat{P}(x)$ and $\hat{P}(y)$ to fulfill the constraint in Eqn.\eqref{eq:high_level_obj}.

To solve the DSL problem, following the common practice in constraint optimization, we introduce Lagrange multipliers and add the equality constraint of probabilistic duality into the objective functions. First, we convert the probabilistic duality constraint into the following regularization term (with the empirical marginal distributions included): 
\begin{equation}
\begin{aligned}
\ell_{\text{duality}} =&(\log\hat{P}(x)+\log P(y|x; \paramAB)\\
-&\log\hat{P}(y) - \log P(x|y; \paramBA))^2.
\end{aligned}
\label{eq:reg_term}
\end{equation}
Then, we learn the models of the two tasks by minimizing the weighted combination between the original loss functions and the above regularization term. The  algorithm is shown in Algorithm \ref{alg:dsl}.
\begin{algorithm}[!htpb]
	\caption{Dual Supervise Learning Algorithm}
	\label{alg:dsl}
	\begin{algorithmic}
		\STATE \textbf{Input}: Marginal distributions $\hat{P}(x_i)$ and $\hat{P}(y_i)$ for any $i\in[n]$; Lagrange parameters $\lambda_{xy}$ and $\lambda_{yx}$; optimizers $Opt_1$ and $Opt_2$;\\
		\REPEAT
		{
			\STATE Get a minibatch of $m$ pairs $\{ (x_j,y_j) \}_{j=1}^{m}$;
			\STATE Calculate the gradients as follows:
			%\begin{small}
			\begin{equation}
			\begin{aligned}
			G_f = \nabla_{\paramAB}&(1/m)\textstyle{\sum_{j=1}^{m}}\big[\ell_1(f(x_j;\paramAB),y_j) \\
			& + \lambda_{xy}\ell_{\text{duality}}(x_j,y_j;\paramAB, \paramBA)\big];\\
			G_g = \nabla_{\paramBA}&(1/m)\textstyle{\sum_{j=1}^{m}}\big[\ell_2(g(y_j;\paramBA),x_j)\\
			& +\lambda_{yx}\ell_{\text{duality}}(x_j,y_j;\paramAB, \paramBA)\big];
			\end{aligned}
			\end{equation}
			%\end{small}
			\STATE Update the parameters of $f$ and $g$:
			\STATE $\paramAB \leftarrow Opt_1(\paramAB, G_f)$, $\paramBA \leftarrow Opt_2(\paramBA, G_g)$.}
		\UNTIL{models converged}
	\end{algorithmic}
\end{algorithm}

In the algorithm, the choice of optimizers $Opt_1$ and $Opt_2$ is quite flexible. One can choose different optimizers such as Adadelta~\cite{zeiler2012adadelta}, Adam~\cite{kingma2014adam}, or SGD for different tasks, depending on common practice in the specific task and personal preferences. 

\subsection{Discussions}
The duality between tasks has been used to enable learning from unlabeled data in \citep{he2016dual}. As an early attempt to exploit the duality, this work actually uses the exterior connection between dual tasks, which helps to form a closed feedback loop and enables unsupervised learning. For example, in the application of machine translation, the primal task/model first translates an unlabeled English sentence $x$ to a French sentence $y^\prime$; then, the dual task/model translates $y^\prime$ back to an English sentence $x^\prime$; finally, both the primal and the dual models get optimized by minimizing the difference between $x^\prime$ with $x$. In contrast, by making use of the intrinsic probabilistic connection between the primal and dual models, DSL  takes an innovative attempt to extend the benefit of duality to supervised learning.

While $\ell_{\text{duality}}$ can be regarded as a regularization term, it is data dependent, which makes DSL different from Lasso~\cite{tibshirani1996regression} or SVM~\cite{hearst1998support}, where the regularization term is data-independent. More accurately speaking, in DSL, every training sample contributes to the regularization term, and each model contributes to the regularization of the other model. 

DSL is different from the following three learning schemes: (1) \emph{Co-training} focuses on single-task learning and assumes that different subsets of features can provide enough and complementary information about data, while DSL targets at learning two tasks with structural duality simultaneously and does not yield any prerequisite or assumptions on features. (2) \emph{Multi-task learning} requires that different tasks share the same input space and coherent feature representation while DSL does not. (3) \emph{Transfer Learning} uses auxiliary tasks to boost the main task, while there is no difference between the roles of two tasks in DSL, and DSL enables them to boost the performance of each other simultaneously.

We would like to point that there are several requirements to apply DSL to a certain scenario: (1) Duality should exist for the two tasks. (2) Both the primal and dual models should be trainable. (3) $\hat{P}(X)$ and $\hat{P}(Y)$ in Eqn.~\eqref{eq:reg_term} should be available. If these conditions are not satisfied, DSL might not work very well. Fortunately, as we have discussed in the paper, many machine learning tasks related to image, speech, and text satisfy these conditions.

\section{Application to Machine Translation}
We first apply our dual supervised learning algorithm to machine translation and study whether it can improve the translation qualities by utilizing the probabilistic duality of dual translation tasks. In the following of the section, we perform experiments on three pairs of dual tasks~\footnote{Since both tasks in each pair are symmetric, they play the same role in the dual supervised learning framework. Consequently, any one of the dual tasks can be viewed as the primal task while the other as the dual task.}: English$\leftrightarrow$French (En$\leftrightarrow$Fr), English$\leftrightarrow$Germany (En$\leftrightarrow$De), and English$\leftrightarrow$Chinese (En$\leftrightarrow$Zh). 

\subsection{Settings}
\noindent\emph{Datasets\ \ } We employ the same datasets as used in~\cite{cho2015using} to conduct experiments on En$\leftrightarrow$Fr and En$\leftrightarrow$De. As a part of WMT’14, the training data consists of $12$M sentences pairs for En$\leftrightarrow$Fr and $4.5$M for En$\leftrightarrow$De, respectively~\cite{wmt14}. We combine \emph{newstest2012} and \emph{newstest2013} together as the validation sets and use \emph{newstest2014} as the test sets. For the dual tasks of En$\leftrightarrow$Zh, we use $10$M sentence pairs obtained from a commercial company as training data. We leverage NIST2006 as the validation set and NIST2008 as well as NIST2012 as the test sets\footnote{The three NIST datasets correspond to Zh$\rightarrow$En translation task, in which each Chinese sentence has four English references. To build the test set for En$\rightarrow$Zh, we use the Chinese sentence with one randomly picked English sentence to form up a En$\rightarrow$Zh validation/test pair. }. Note that, during the training of all three pairs of dual tasks, we drop all sentences with more than $50$ words.

\noindent\emph{Marginal Distributions $\hat{P}(x)$ and $\hat{P}(y)$\ \ } We use the LSTM-based language modeling approach~\cite{sundermeyer2012lstm,mikolov2010recurrent} to characterize the marginal distribution of a sentence $x$, defined as $\prod_{i=1}^{T_x}P(x_i|x_{<i})$, where $x_i$ is the $i$th word in $x$, $T_x$ denotes the number of words in $x$, and  the index $<i$ indicates $\{1,2,\cdots,i-1\}$. More details about such language modeling approach can be referred to Appendix~\ref{app:LM_details}.
%$\mathbb{P}\{\text{sentence}\}$. The embedding dimension and hidden node are both $1024$. We apply 0.5 dropout to the input embedding and the last hidden layer before softmax. The vocabulary sizes are $30$k for the two language models of En$\leftrightarrow$Fr and $50$k for En$\leftrightarrow$De, as well as $30$k for En$\leftrightarrow$Zh. %The validation/test perplexities are $88.72/94.22$ (En$\leftrightarrow$Fr, En), $58.90/57.68$ (En$\leftrightarrow$Fr, Fr), $??$ (En$\leftrightarrow$De, En), $84.81$ (En$\leftrightarrow$De, De), XX/XX (En$\leftrightarrow$Zh, En) and XX/XX (En$\leftrightarrow$Zh, Zh). Note The validation sets for the language model training and translation model training are the same. 

\noindent\emph{Model}  We apply the GRU as the recurrent module to implement the sequence-to-sequence model, which is the same as~\cite{bahdanau2015neural,cho2015using}. The word embedding dimension is $620$ and the number of hidden node is $1000$. Regarding  the vocabulary size of the source and target language, we set it as $30$k, $50$k, and $30$k for En$\leftrightarrow$Fr, En$\leftrightarrow$De, and En$\leftrightarrow$Zh, respectively. The out-of-vocabulary words are replaced by a special token UNK. Following the common practice, we denote the baseline algorithm proposed in \cite{bahdanau2015neural,cho2015using} as \emph{RNNSearch}. We implement the whole NMT learning system based on an open source code\footnote{https://github.com/nyu-dl/dl4mt-tutorial}.

\noindent\emph{Evaluation Metrics\ \ } The translation qualities are measured by tokenized case-sensitive BLEU~\cite{papineni2002bleu}  scores, which is implemented by~\cite{multibleu}. The larger the BLEU score is, the better the translation quality is. During the evaluation process, we use beam search with beam width 12 to generate sentences. Note that, following the common practice, the Zh$\rightarrow$En is evaluated by case-insensitive BLEU score.
%Besides BLEU, we also evaluate our models by the test perplexities. The perplexity is the geometric mean of the inverse probability for each predicted word, which is closely related to the objective function. 

\noindent\emph{Training Procedure} We initialize the two models in DSL (i.e., the $\paramAB$ and $\paramBA$) by using two warm-start models, which is generated by following the same process as~\cite{cho2015using}. Then, we use SGD with the minibatch size of 80 as the optimization method for dual training. During the training process, we first set the initial learning rate $\eta$ to $0.2$ and then halve it if the BLEU score on the validation set cannot grow for a certain number of mini batches. In order to stabilize parameters, we will freeze the embedding matrix once halving learning rates can no long improve the BLEU score on the validation set. The gradient clip is set as $1.0$, $5.0$ and $1.0$ during the training for En$\leftrightarrow$Fr, En$\leftrightarrow$De, and En$\leftrightarrow$Zh, respectively~\cite{pascanu2013difficulty}. The value of both $\lambda_{xy}$ and $\lambda_{yx}$ in Algorithm \ref{alg:dsl} are set as $0.01$ according to empirical performance on the validation set. Note that, during the optimization process, the LSTM-based language models will not be updated.

\subsection{Results}
Table \ref{tab:overall_statistics} shows the BLEU scores on the dual tasks by the DSL method with that by the baseline RNNSearch method. Note that, in this table, we use (MT08) and (MT12) to denote results carried out on NIST2008 and NIST2012, respectively. From this table, we can find that, on all these three pairs of symmetric tasks, DSL can improve the performance of both dual tasks, simultaneously. 

\begin{table}[!htpb]
	\centering
	\caption{BLEU scores of the translation tasks. ``$\Delta$'' represents the improvement of DSL over RNNSearch.}
	\begin{tabular}{|c|c|c|c|}
		\hline
		Tasks & RNNSearch & DSL & $\Delta$ \\
		\hline
		En$\rightarrow$Fr & $29.92$  & $31.99$ & $2.07$ \\ 
		\hline
		Fr$\rightarrow$En & $27.49$  & $28.35$ & $0.86$\\ 
		\hline
		En$\rightarrow$De & $16.54$  & $17.91$ & $1.37$   \\ 
		\hline 
		De$\rightarrow$En & $20.69$  & $20.81$ & $0.12$   \\ 
		\hline
		En$\rightarrow$Zh (MT08) & $15.45$  & $15.87$ & $0.42$   \\ 
		\hline
		Zh$\rightarrow$En (MT08) & $31.67$  & $33.59$ & $1.92$   \\ 
		\hline
		En$\rightarrow$Zh (MT12) & $15.05$  & $16.10$ & $1.05$   \\ 
		\hline
		Zh$\rightarrow$En (MT12) & $30.54$  & $32.00$ & $1.46$   \\ 
		\hline
	\end{tabular}
	\label{tab:overall_statistics}
\end{table}

To better understand the effects of applying the probabilistic duality constraint as the regularization, we compute the $\ell_\text{duality}$ on the test set by DSL compared with RNNSearch. In particular, after applying DSL to En$\rightarrow$Fr, the $\ell_\text{duality}$ decreases from $1545.68$ to $1468.28$, which also indicates that the two models become more coherent in terms of probabilistic duality.

\cite{cho2015using} proposed an effective post-process technique, which can achieve better translation performance by replacing the ``UNK'' with the corresponding word-level translations. After applying this technique into DSL, we report its results on En$\rightarrow$Fr in Table \ref{tab:statistics_unk}, compared with several baselines with the same model structures as ours that also integrate the ``UNK'' post-processing technique. From this table, it is clear to see that DSL can achieve better performance than all baseline methods.
\begin{table}[!htbp]
	\centering
	\caption{Summary of some existing En$\rightarrow$Fr translations}
	\begin{tabular}{|c|c|c|}
		\hline
		Model & Brief description & BLEU \\
		\hline
		NMT[1] & \emph{standard NMT} & $33.08$ \\
		%\hline
%		NMT-LV[2] & \emph{vocabulary size = $500$K} & $34.11$ \\
		\hline 
		MRT[2] & \emph{Direct optimizing BLEU} & $34.23$ \\
		\hline
%		dual-NMT[4] & \emph{ with unlabeled data} & $34.83$ \\
%		\hline
		\hline
		DSL & \emph{ Refer to Algorithm \ref{alg:dsl}} & $\bm{34.84}$ \\
	%	\hline
	%	DSL + [4] & \emph{use} [4] \emph{as warm-start} & $\bm{35.62}$ \\
		\hline
		\hline
		\multicolumn{3}{|l|}{[1]~\cite{cho2015using}; [2]~\cite{shen2016minimum}}\\
		%\multicolumn{3}{|l|}{[4]~\cite{he2016dual}}\\
		\hline
	\end{tabular}
	\label{tab:statistics_unk}
\end{table}

In the previous experiments, we use a warm-start approach in DSL using the models trained by RNNSearch. Actually, we can use stronger models for initialization to achieve even better accuracy. We conduct a light experiment to verify this. We use the models trained by~\cite{he2016dual} as the initializations in DSL on En$\leftrightarrow$Fr translation. We find that BLEU score can be improved from $34.83$ to $\bm{35.95}$ for En$\rightarrow$Fr translation, and from $32.94$ to $\bm{33.40}$ for Fr$\rightarrow$En translation. 
\begin{figure}[!b]
	\centering
	\begin{minipage}{0.5\linewidth}
		\subfigure[Valid BLEU w.r.t $\lambda$]{
			\includegraphics[width=\linewidth]{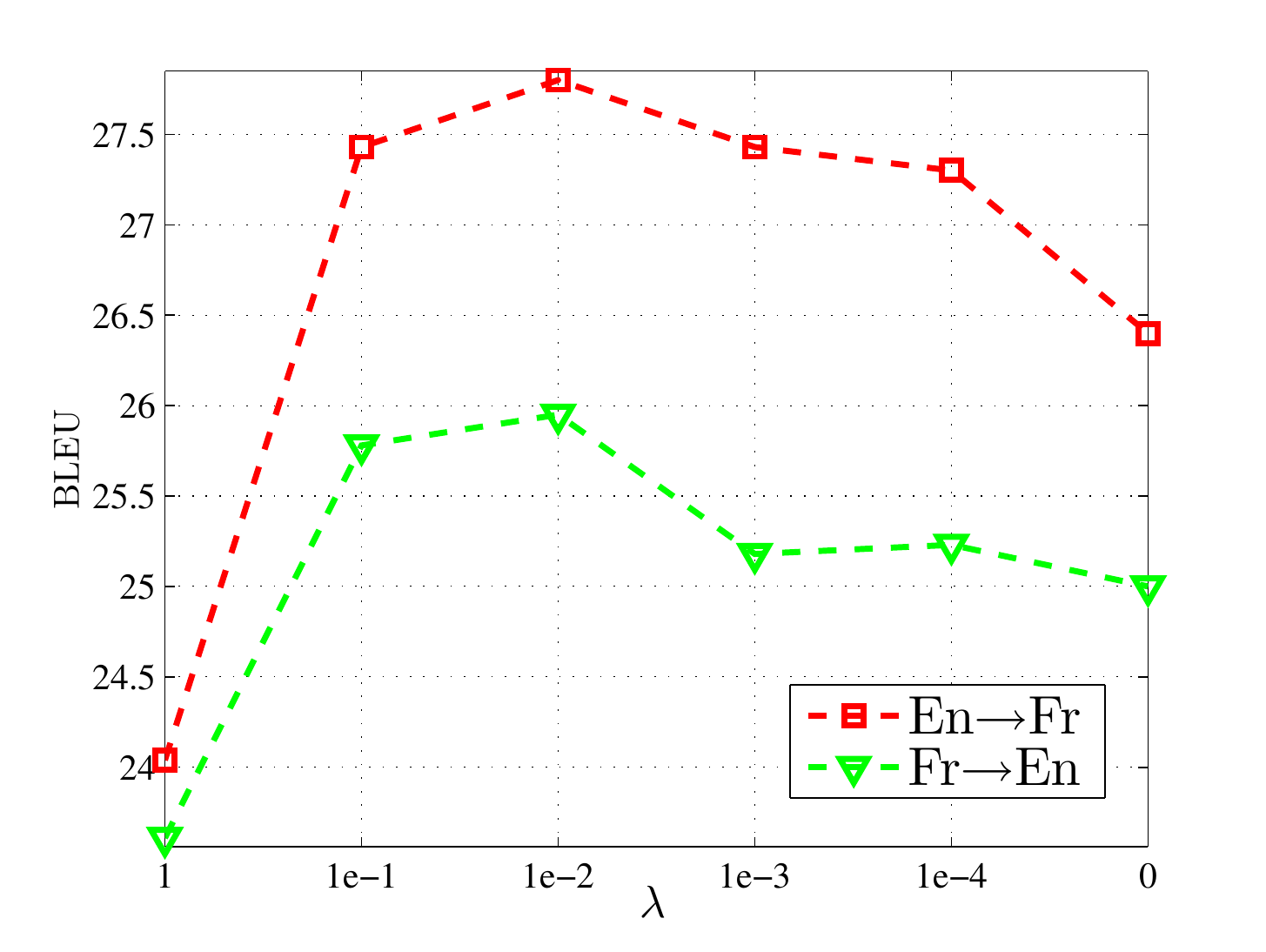}	
		}
	\end{minipage}%
	\begin{minipage}{0.5\linewidth}
		\subfigure[Valid / Test BLEU curves]{
			\includegraphics[width=\linewidth]{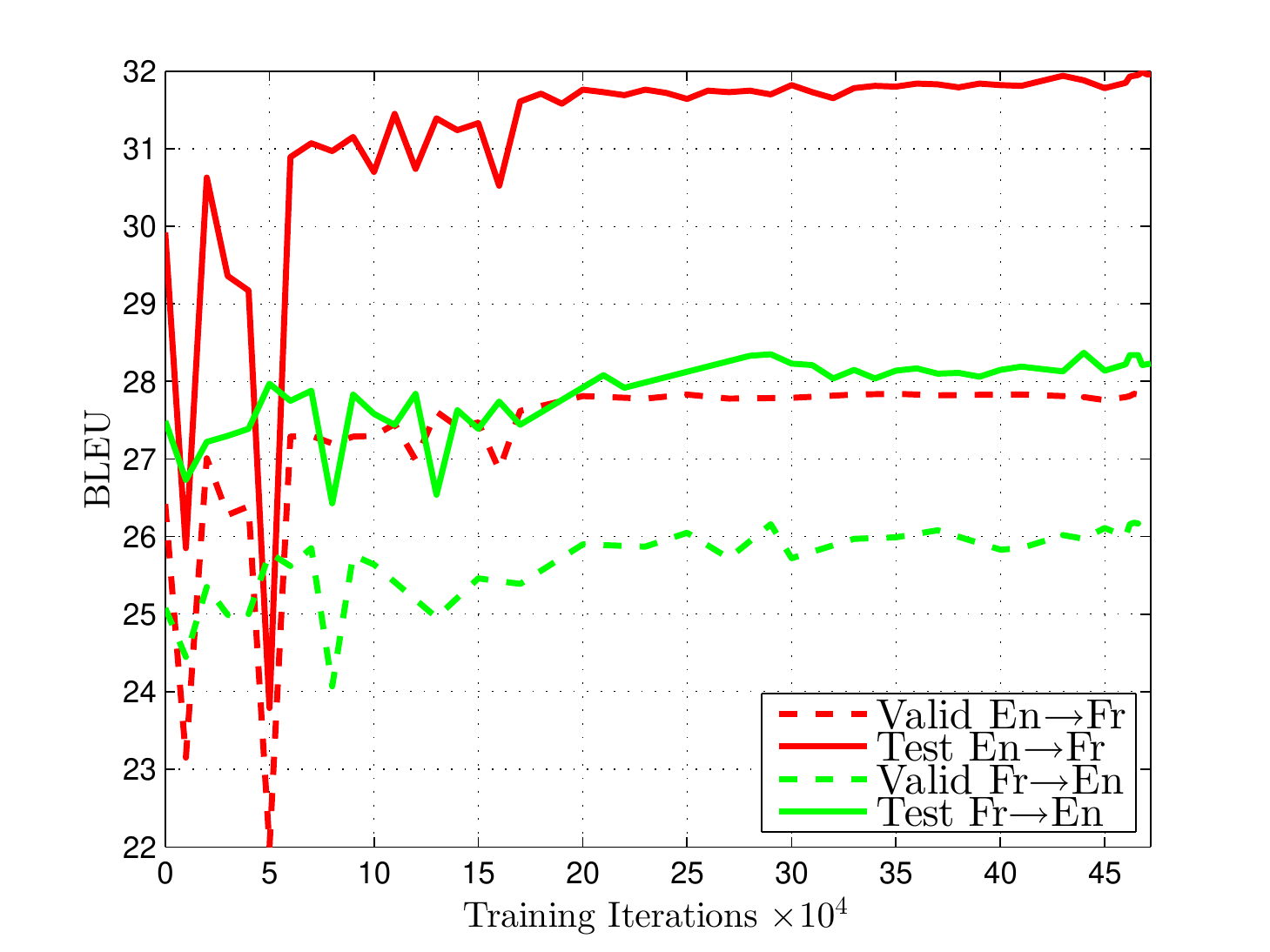}	
		}
	\end{minipage}
	\caption{Visualization of En$\leftrightarrow$Fr tasks with DSL}
	\label{fig:en-fr-visulaziation}
\end{figure}

\noindent{\bf Effects of $\lambda$}\\
There are two hyperparameters $\lambda_{xy}$ and $\lambda_{yx}$ in our DSL algorithm. We conduct some experiments to investigate their effects. Since the input and output space are symmetric,  we set $\lambda_{xy}=\lambda_{yx}=\lambda$ and plot the validation accuracy of different $\lambda$'s in Figure \ref{fig:en-fr-visulaziation}(a). From this figure, we can see that both En$\rightarrow$Fr and Fr$\rightarrow$En reach the best performance when $\lambda=10^{-2}$, and thus the results of DSL reported in Table \ref{tab:overall_statistics} are obtained with $\lambda=10^{-2}$. Moreover, we find that, within a relatively large interval of $\lambda$, DSL outperforms standard supervised learning, i.e., the point with $\lambda=0$. We also plot the BLEU scores for $\lambda=10^{-2}$ on the validation and test sets in Figure \ref{fig:en-fr-visulaziation}(b) with respect to training iterations. We can see that, in the first couple of rounds, the test BLEU curves fluctuate with large variance. The reason is that two separately initialized models of dual tasks yield are not consistent with each other, i.e., Eqn. (1) does not hold, which causes the declination of the performance of both models as they play as the regularizer for each other. As the training goes on, two models become more consistent and finally boost the performance of each other.

\begin{table}[!t]
	\centering
	\caption{An example of En$\leftrightarrow$Fr}
	\begin{tabular}{|l|}
		\hline
		{\small{{\bf[Source (En)]} \emph{A board member at a German blue-chip}}}\\ 
		{\small{\emph{company concurred that when it comes to economic espionage,}}}\\
		{\small{\emph{"the French are the worst."}}} \\
		\hline
		{\small{{\bf [Source (Fr)]} \emph{Un membre du conseil d'administration d'une}}}\\
		{\small{\emph{société allemande renommée estimait \textbf{que lorsqu'il s'agit}}}}\\
		{\small{\emph{d'espionnage économique , « les Français sont les pires » .}}} \\
		\hline
		{\small{{\bf [RNNSearch (Fr$\rightarrow$En)]} \emph{A member of the board of directors}}}\\
		{\small{\emph{of a renowned \textbf{German society} felt that \textbf{when it was} economic}}}\\
		{\small{\emph{ espionage, “the French are the worst. ”}}}\\
		\hline
		{\small{{\bf [RNNSearch (En$\rightarrow$Fr)]} \emph{Un membre du conseil d'une}}}\\
		{\small{\emph{compagnie allemande UNK a reconnu que quand il s'agissait}}}\\
		{\small{\emph{d'espionnage économique, "\textbf{le français est le pire}".}}} \\
		\hline
		{\small{{\bf [DSL (Fr$\rightarrow$En)]} \emph{A board member of a renowned \textbf{German}}}}\\
		{\small{\emph{\textbf{company} felt that \textbf{when it comes to} economic espionage,}}}\\
		{\small{\emph{"the French are the worst. "}}} \\
		\hline
		{\small{{\bf [DSL (En$\rightarrow$Fr)]} \emph{Un membre du conseil d'une compagnie}}}\\
		{\small{\emph{allemande UNK a reconnu que , \textbf{lorsqu'il s'agit d'}espionnage}}}\\
		{\small{\emph{économique, "\textbf{les Français sont les pires}".}}} \\
		\hline
	\end{tabular}
	\label{tab:en-fr-caseStudy}
\end{table}

\noindent{\bf Case studies}\\
Table \ref{tab:en-fr-caseStudy} shows a couple of translation examples produced by RNNSearch compared with DSL. From this table, we  find that DSL demonstrates three major advantages over RNNSearch. First, by leveraging the structural duality of sentences, DSL can result in the improvement of mutual translation, e.g. ``\emph{when it comes to}'' and ``\emph{lorsqu qu’il s’agit de}'', which better fit the semantics expressed in the sentences. Second, DSL can consider more contextual information in translation. For example, in Fr$\rightarrow$En, \emph{une société} is translated to company, however, in the baseline, it is translated to society. Although the word level translation is not bad, it should definitely be translated as ``company'' given the contextual semantics. Furthermore, DSL can better handle the plural form. For example,  DSL can correctly translate ``\emph{the French are the worst}'', which are of plural form, while the baseline deals with it by singular form.

\section{Application to Images Processing}
In the domain of image processing, image classification (image$\rightarrow$label) and image generation (label$\rightarrow$image) are in the dual form. In this section, we apply our dual supervised learning framework to these two tasks and conduct experimental studies based on a public dataset, CIFAR-10~\cite{krizhevsky2009learning}, with 10 classes of images. In our experiments, we employ a popular method, ResNet\footnote{https://github.com/tensorflow/models/tree/master/resnet}, for image classification and a most recent method, PixelCNN++\footnote{https://github.com/openai/pixel-cnn}, for image generation. Let $\mathcal{X}$ denote the image space and $\mathcal{Y}$ denote the category space related to CIFAR-10.

\subsection{Settings}
\emph{Marginal Distributions\ \ } In our experiments, we simply use the uniform distribution to set the marginal distribution $\hat{P}(y)$ of 10-class labels, which means the marginal distribution of each class equals $0.1$. The image distribution $\hat{P}(x)$ is usually defined as $\prod_{i=1}^{m}P\{x_i|x_{<i}\}$, where all pixels of the image is serialized and $x_i$ is the value of the $i$-th pixel of an $m$-pixel image. Note that the model can predict $x_i$ only based on the previous pixels $x_j$ with index $j<i$. We use the PixelCNN++, which is so far the best algorithm, to model the image distribution.  %The \emph{bpd} (bits per dimension, details of which will be described later in this section) of the model we use to characterize the marginal distribution of images is $2.92$. 

\emph{Models\ \ } For the task of image classification, we choose 32-layer ResNet (denoted as ResNet-32) and 110-layer ResNet (denoted as ResNet-110) as two baselines, respectively, in order to examine the power of DSL on both relatively simple and complex models. For the task of image generation, we use PixelCNN++ again. Compared to the PixelCNN++ used for modeling distribution, the difference lies in the training process: When used for image generation given a certain class, PixelCNN++ takes the class label as an additional input, i.e., it tries to characterize $\prod_{i=1}^{m}\mathbb{P}\{x_i|x_{<i},y\}$, where $y$ is the 1-hot label vector.

\emph{Evaluation Metrics\ \ } We use the classification error rates to measure the performance of image classification. We use \emph{bits per dimension} (briefly, \emph{bpd})~\cite{tim2017pixel}, to assess the performance of image generation. In particular, for an image $x$ with label $y$, the \emph{bpd} is defined as: 
\begin{equation}
-\big(\textstyle{\sum}_{i=1}^{N_x}\log P(x_i|x_{<i},y)\big)/\big(N_x\log(2)\big),
\end{equation}
where $N_x$ is the number of pixels in image $x$. By using the dataset CIFAR-10, $N_x$ is $3072$ for any image $x$, and we will report the average \emph{bpd} on the test set. 

\emph{Training Procedure\ \ } We first initialize both the primal and the dual models with the ResNet model and PixelCNN++ model pre-trained independently and separately. We obtain a 32-layer ResNet with error rate of $7.65$ and a 110-layer ResNet with error rate of $6.54$ as the pre-trained models for image classification. The error rates of these two pre-trained models are comparable to results reported in~\cite{he2015deep}. We generate a pre-trained conditional image generation model with the test \emph{bpd} of $2.94$, which is the same as reported in~\cite{tim2017pixel}. For DSL training, we set the initial learning rate of image classification model as $0.1$ and that of image generation model as $0.0005$. The learning rates follow the same decay rules as those in~\cite{he2015deep} and~\cite{tim2017pixel}. The whole training process takes about two weeks before convergence. Note that experimental results below are based on the training with $\lambda_{xy}=(30/3072)^2$ and $\lambda_{yx}=(1.2/3072)^2$. %according to the validation performance\footnote{We first randomly sample $5k$ images from the training set as the validation set for hyperparameter tuning and remaining training data for training. After identifying a set of good hyperparameters, we fix the hyperparameters and use all the training set for training.}.

\subsection{Results on Image Classification}
Table~\ref{tab:resnet_classification_a} compares the error rates of two image classification models, i.e., DSL vs. Baseline, on the test set. From this table, we find that, with using either ResNet-32 or ResNet-110, DSL achieves better accuracy than the baseline method. 

\begin{table}[!t]
	\centering
	\caption{Error rates (\%) of image classification tasks. Baseline is from~\cite{he2015deep}. ``$\Delta$'' denotes the improvement of DSL over baseline.}
	\begin{tabular}{|c|c|c|}
		\hline
		  & ResNet-32 & ResNet-110  \\
		\hline
		baseline & $7.51$ & $6.43$ \\
		\hline
		DSL & $6.82$ & $5.40$\\
		\hline
		$\Delta$ & $0.69$ & $1.03$\\
		\hline
	\end{tabular}
	\label{tab:resnet_classification_a}
\end{table}

Interestingly, we observe from Table~\ref{tab:resnet_classification_a} that, DSL leads to higher relative performance improvement on the ResNet-110 over the ResNet-32. We hypothesize one possible reason is that, due to the limited training data, an appropriate regularization can benefit more to the 110-layer ResNet with higher model complexity, and the duality-oriented regularization $\ell_{\text{duality}}$ indeed plays this role and consequently gives rise to higher relative improvement.

\subsection{Results on Image Generation}
Our further experimental results show that, based on ResNet-110, DSL can decrease the test \emph{bpd} from $2.94$ (baseline) to  $2.93$ (DSL), which is a new state-of-the-art result on CIFAR-10. Indeed, it is quite difficult to improve \emph{bpd} by $0.01$ which though seems like a minor change. We also find that, there is no significant improvement on test \emph{bpd} based on ResNet-32. An intuitive explanation is that, since ResNet-110 is stronger than ResNet-32 in modeling the conditional probability $P(y|x)$, it can better help the task of image generation through the constraint/regularization of the probabilistic duality. 

As pointed out in \cite{theis2015note}, \emph{bpd} is not the only evaluation rule of image generation. Therefore, we further conduct a qualitative analysis by comparing images generated by dual supervised learning with those by the baseline model for each of image categories, some examples of which are shown in Figure \ref{fig:imgGen}. 
\begin{figure}[!htpb]
	\centering
	\includegraphics[width=0.95\linewidth]{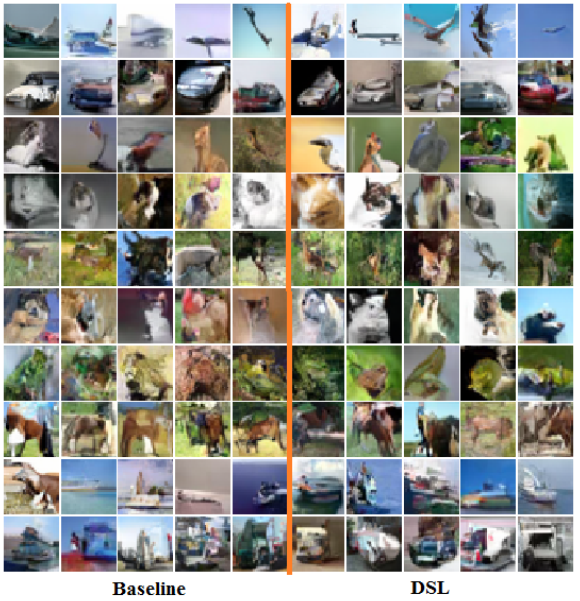}
	\caption{Generated images}
	\label{fig:imgGen}
\end{figure}

Each row in Figure \ref{fig:imgGen} corresponds to one category in CIFAR-10, the five images in the left side are generated by the baseline model, and the five ones in the right side are generated by the model trained by DSL. From this figure, we find that DSL generally generates images with clearer and more distinguishable characteristics regarding the corresponding category. Specifically, those right five images in Row 3, 4, and 6 can illustrate more distinguishable characteristics of birds, cats and dogs respectively, which is mainly due to benefits of introducing the probabilistic duality into DSL. But, there are still some cases that neither the baseline model nor DSL can perform well, like deers it Row 5 and frogs in Row 7. One reason is that the \emph{bpd} of images in the category of deer and frogs are $3.17$ and $3.32$, which are significant larger than the average $2.94$. This shows that the images of these two categories are harder to generate.

\section{Application to Sentiment Analysis}
Finally, we apply the dual supervised learning framework to the domain of sentiment analysis. In this domain, the primal task, sentiment classification~\cite{maas2011learning,dai2015semi}, is to predict the sentiment polarity label of a given sentence; and the dual task, though not quite apparent but really existed, is sentence generation based on a sentiment polarity. In this section, let $\mathcal{X}$ denote the sentences and $\mathcal{Y}$ denote the sentiment related to our task.

\subsection{Experimental Setup}
\noindent\emph{Dataset\ \ } Our experiments are performed based on the IMDB movie review dataset~\cite{IMDB}, which consists of $25$k training and $25$k test sentences. Each sentence in this dataset is associated with either a positive or a negative sentiment label. We randomly sample a subset of $3750$ sentences from the training data as the validation set for hyperparameter tuning and use the remaining training data for model training. 

\noindent\emph{Marginal Distributions\ \ } We simply use the uniform distribution to set the marginal distribution $\hat{P}(y)$ of polarity labels, which means the marginal distribution of positive or negative class equals $0.5$. On the other side, we take advantage of the LSTM-based language modeling to model the marginal distribution $\hat{P}(x)$ of a sentence $x$. The test perplexities~\cite{bengio2003neural} of the obtained language model is $58.74$. 

\noindent\emph{Model Implementation\ \ } We leverage the widely used LSTM~\cite{dai2015semi} modeling approach for sentiment classification\footnote{Both supervised and semi-supervised sentiment classification are studied in ~\cite{dai2015semi}. We focus on supervised learning here. Therefore, we do not compare with the models trained with semi-supervised (labeled + unlabeled) data.} model. We set the embedding dimension as $500$ and the hidden layer size as $1024$. For sentence generation, we use another LSTM model with $W^e_wE_wx_{t-1}+W^e_sE_sy$ as input, where $x_{t-1}$ denotes the $t-1$'th word, $E_w$ and $E_s$ represent the embedding matrices for word and sentiment label respectively, and $W$'s represent the connections between embedding matrix and LSTM cells. A sentence is generated word by word sequentially, and the probability that word $x_t$ is generated is proportional to $\exp(W^d_wE_wx_{t-1}+W^d_sE_sy+W_hh_{t-1})$, where $h_{t-1}$ is the hidden state outputted by LSTM. Note the $W$'s and the $E$'s are the parameters to learn in training. In the following, we call the model for sentiment based sentence generation as contextual language model (briefly, CLM).%, where the contextual information refers to the sentiment label here\footnote{The machine translation network and PixelCNN++ can both be regarded as CLMs.}. 

\noindent\emph{Evaluation Metrics\ \ } We measure the performance of sentiment classification by the error rate, and that of sentence generation, i.e., CLM, by test perplexity.

\noindent\emph{Training Procedure\ \ } To obtain baseline models, we use Adadelta as the optimization method to train both the sentiment classification and sentence generation model. Then,  we use them to initialization the two models for DSL. At the beginning of DSL training, we use plain SGD with an initial learning rate of $0.2$ and then decrease it to $0.02$ for both models once there is no further improvement on the validation set. For each $(x,y)$ pair, we set $\lambda_{xy}=(5/l_x)^2$ and $\lambda_{yx}=(0.5/l_x)^2$, where $l_x$ is the length of $x$. The whole training process of DSL takes less than two days.

\subsection{Results}
Table \ref{tab:imdb} compares the performance of DSL with the baseline method in terms of both the error rates of sentiment classification and the perplexity of sentence generation. Note that the test error of the baseline classification model, which is $10.10$ as shown in the table, is comparable to the recent results as reported in \cite{dai2015semi}. We have two observations from the table. First, DSL can reduce the classification error by $0.90$ without modifying the LSTM-based model structure. Second, DSL slightly improves the perplexity for sentence generation, but the improvement is not very significant.  We hypothesize the reason is that the sentiment label can merely supply at most 1 bit  information such that the perplexity difference between the language model (i.e., the marginal distribution $\hat{P}(x)$) and CLM (i.e., the conditional distribution $P(x|y)$) are not large, which limits the improvement brought by DSL.

\begin{table}[!htpb]
	\centering
	\caption{Results on IMDB}
	\begin{tabular}{|c|c|c|}
		\hline
		& Test Error (\%) & Perplexity \\
		\hline
		Baseline & $10.10$ & $59.19$ \\
		\hline 
		DSL & $9.20$ & $58.78$ \\
		\hline
	\end{tabular}
	\label{tab:imdb}
\end{table}

\noindent\textbf{Qualitative analysis on sentence generation}

In addition to quantitative studies as shown above, we further conduct qualitative analysis on the performance of sentence generation. Table \ref{tab:clm_case} demonstrates some examples of generated sentences based on sentiment labels.
From this table, we can find that both the baseline model and DSL succeed in generating sentences expressing the certain sentiment. The baseline model prefers to produce the sentence with those words yielding high-frequency in the training data, such as the ``\emph{the plot is simple/predictable, the acting is great/bad}'', etc. This is because the sentence generation model itself is essentially a language model based generator, which aims at catching the high-frequency words in the training data. Meanwhile, since the training of CLM in DSL can leverage the signals provided by the classifier, DSL makes it more possible to select those words, phrases, or textual patterns that can present more specific and more intense sentiment, such as ``\emph{nothing but good, 10/10, don't waste your time}'', etc. As a result, the CLM in DSL can generate sentences with richer expressions for sentiments.

\begin{table}[!t]
	\centering
	\caption{Sentence generation with given sentiments}
	\begin{tabular}{|c|l|}
		\hline
		& {\small{\emph{i've seen this movie a few times. it's still one of my}}}\\
		\small{\emph{Base}} &{\small{\emph{favorites. the plot is simple, the acting is great.}}}\\
		\small{\emph{(Pos)}}&{\small{\emph{It's a very good movie, and i think it's one of the}}}\\
		&{\small{\emph{best movies i've seen in a long time.}}} \\
		\hline
		& {\small{\emph{\textbf{I have nothing but good things to say about this} }}}\\
		&{\small{\emph{\textbf{movie}. I saw this movie when it first came out,}}}\\
		\small{\emph{DSL}} &{\small{\emph{and I had to watch it again and again. I really}}}\\
		\small{\emph{(Pos)}} &{\small{\emph{enjoyed this movie. I thought it was a very good}}}\\
		&{\small{\emph{movie. The acting was great, the story was great.}}}\\
		&{\small{\emph{\textbf{I would recommend this movie to anyone.}}}} \\
		&{\small{\emph{\textbf{I give it 10 / 10.}}}}\\
		\hline
		&{\small{\emph{after seeing this film, i thought it was going to be}}}\\
		\small{\emph{Base}}&{\small{\emph{one of the worst movies i've ever seen; the acting}}}\\
		\small{\emph{(Neg)}}&{\small{\emph{was bad, the script was bad. the only thing i can}}}\\
		&{\small{\emph{say about this movie is that it's so bad.}}}\\
		\hline
		& {\small{\emph{this is a difficult movie to watch, and would, \textbf{not}}}}\\
		\small{\emph{DSL}}&{\small{\emph{\textbf{recommend it to anyone}. The plot is predictable,}}}\\
		\small{\emph{Neg}}&{\small{\emph{the acting is bad, and the script is awful. }}}\\
		&{\small{\emph{\textbf{Don't waste your time on this one.}}}}\\
		\hline
	\end{tabular}
	\label{tab:clm_case}
\end{table}

\subsection{Discussions}
In previous experiments, we start DSL training with well-trained primal and dual models. We conduct some further experiments to verify whether warm start is a must for DSL. (1) We train DSL from a warm-start sentence generator and a cold-start (randomly initialized) sentence classifier. In this case, DSL achieves a classification error of $9.44\%$, which is better than the baseline classifier in Table~\ref{tab:imdb}. (2) We train DSL from a warm-start classifier and a cold-start sentence generator. The perplexity of the generator after DSL training reach $58.79$, which is better than the baseline generator. (3) We train DSL from both cold-start models. The final classification error is $9.50$\% and the perplexity of the generator is 58.82, which are both better than the baselines. These results show that the success of DSL does not necessarily require warm-start models, although they can speed up the training of DSL.  

\section{Conclusions and Future Work}
Observing the existence of structure duality among many AI tasks, we have proposed a new learning framework, dual supervised learning, which can greatly improve the performance for both the primal and the dual tasks, simultaneously. We have introduced a probabilistic duality term to serve as a data-dependent regularizer to better guide the training. Empirical studies have validated the effectiveness of dual supervised learning.

There are multiple directions to explore in the future. First, we will test dual supervised learning on more dual tasks, such as speech recognition and speech synthesis. Second, we will enrich theoretical study to better understand dual supervised learning.  Third, it is interesting to combine dual supervised learning with unsupervised dual learning~\cite{he2016dual} to leverage unlabeled data so as to further improve the two dual tasks. Fourth, we will combine dual supervised learning with dual inference~\cite{DI2017xia} so as to leverage structural duality to enhance both the training and inference procedures.

\appendix
{\Large{\em \bf Appendix}}	
\section{Theoretical Analysis}\label{app:theory_analysis}

As we know, the final goal of the dual learning is to give correct predictions for the unseen test data. That is to say, we want to minimize the (expected) risk of the dual models, which is defined as follows\footnote{The parameters $\paramAB$ and $\paramBA$ in the dual models will be omitted when the context is clear.}: 
\begin{equation*}
		R(f,g) = \mathbb{E}\left[\frac{\ell_1(f(x),y) +\ell_2(g(y),x)}{2} \right], \forall f\in\mathcal{F}, g\in\mathcal{G},
\end{equation*}
where $\mathcal{F}=\{ f(x;\paramAB); \paramAB\! \in\! \Theta_{xy} \}$, $\mathcal{G}=\{ g(x;\paramBA); \paramBA\in \Theta_{yx}\}$, $\Theta_{xy}$ and $\Theta_{yx}$ are parameter spaces, and the $\mathbb{E}$ is taken over the underlying distribution $P$. Besides, let $\mathcal{D}$ denote the product space of the two models satisfying probabilistic duality, i.e., the constraint in Eqn.(4). For ease of reference, define $\mathcal{H}_{\text{dual}}$ as $(\mathcal{F}\times\mathcal{G})\cap\mathcal{D}$.
	
Define the empirical risk on the $n$ sample as follows: for any $f\in\mathcal{F}, g\in\mathcal{G}$,
	\begin{equation*}
		R_n(f,g) = \frac{1}{n} {\sum}_{i=1}^{n}\frac{\ell_1(f(x_i),y_i) +\ell_2(g(y_i),x_i)}{2}.%, \forall f\in\mathcal{F}, g\in\mathcal{G},
	\end{equation*}
	
	Following~\cite{bartlett2002rademacher}, we introduce Rademacher complexity for dual supervised learning, a measure for the complexity of the hypothesis.
	\begin{definition}
		Define the Rademacher complexity of DSL, $\mathfrak{R}_n^{\text{DSL}}$, as follows:
		\begin{small}
			\begin{equation*}
				\mathfrak{R}_n^{\text{DSL}} = \underset{\bm{z},\sigma}{\mathbb{E}}\Big[ \sup_{(f,g)\in\mathcal{H}_{\text{dual}}}\big|\frac{1}{n}\sum_{i=1}^{n}\sigma_i\big(\ell_1(f(x_i),y_i) + \ell_2(g(y_i),x_i)\big) \big|  \Big],
			\end{equation*}
		\end{small}
		where $\bm{z}=\{z_1,z_2,\cdots,z_n\}\sim P^n$, $z_i=(x_i,y_i)$ in which $x_i\in\mathcal{X}$ and $y_i\in\mathcal{Y}$,  $\bm{\sigma}=\{\sigma_1,\cdots,\sigma_m \}$ are i.i.d sampled with $P(\sigma_i=1)=P(\sigma_i=-1)=0.5$.
		\label{def:ra_dsl}
	\end{definition}
	Based on $\mathfrak{R}_n^{\text{DSL}}$, we have the following theorem for dual supervised learning:
	\begin{theorem}[\cite{mohri2012foundations}] 
		Let $\frac{1}{2}\ell_1(f(x),y)+ \frac{1}{2}\ell_2(g(y),x)$ be a mapping from $\mathcal{X}\times\mathcal{Y}$ to $[0,1]$. Then, for any $\delta\in(0,1)$, with probability at least $1-\delta$, the following inequality holds for any $(f,g)\in\mathcal{H}_{\text{dual}}$,
		\begin{equation}
			R(f,g)\le R_n(f,g)+	2\mathfrak{R}_n^{\text{DSL}} + \sqrt{\frac{1}{2n}\ln(\frac{1}{\delta})}.
		\end{equation}	
		\label{thm:estimation_error}
		\vspace{-0.2in} 
	\end{theorem}
	
	Similarly, we define the Rademacher complexity for the standard supervised learning $\mathfrak{R}_n^{\text{SL}}$ under our framework by replacing the $\mathcal{H}_{\text{dual}}$ in Definition \ref{def:ra_dsl} by $\mathcal{F}\times\mathcal{G}$. With probability at least $1-\delta$, the generation error bound of supervised learning is smaller than $2\mathfrak{R}_n^{\text{SL}} + \sqrt{\frac{1}{2n}\ln(\frac{1}{\delta})}$.
	
	Since $\mathcal{H}_{\text{dual}}\in\mathcal{F}\times\mathcal{G}$, by the definition of Rademacher complexity, we have $\mathfrak{R}_n^{\text{DSL}}\le\mathfrak{R}_n^{\text{SL}}$. Therefore, DSL enjoys a smaller generation error bound than supervised learning.

	The approximation of dual supervised learning is defined as
	\begin{equation}
		\begin{aligned}
			R(f_{\mathcal{F}}^*,g_{\mathcal{F}}^*) - R^*
			\label{eq:error_decoms}
		\end{aligned}
	\end{equation}
	in which
	\begin{align*}
		%& (f_n,g_n) = \arg\min_{f,g}\{\textstyle{\sum_{i=1}^N}[\ell_1(f(x_i,y_i)) +\ell_2(g(y_i,x_i))] \}, \nonumber \\
		%&\qquad\quad \text{s.t.}\;\; (f,g)\in\mathcal{H}_{\text{dual}}\ \   \text{where }\mathcal{H}_{\text{dual}}=(\mathcal{F}\times\mathcal{G})\cap \mathcal{D};\label{eq:erm}\\
		& R(f^*_{\mathcal{F}},g^*_{\mathcal{F}}) = \inf R(f,g),\;s.t.\;(f,g)\in\mathcal{H}_{\text{dual}};\\
		& R^* = \inf R(f,g).
	\end{align*}
	The approximation error for supervised learning is similarly defined.
	
	Define $\mathcal{P}_{y|x}=\{P(y|x;\paramAB)| \paramAB\in\Theta_{xy}   \}$, \\
	$\mathcal{P}_{x|y}=\{P(x|y;\paramBA)| \paramBA\in\Theta_{yx}   \}$. Let $P^*_{y|x}$ and $P^*_{x|y}$ denote the two conditional probabilities derived from $P$. We have the following theorem:
	\begin{theorem}
		If $P^*_{y|x}\in\mathcal{P}_{y|x}$ and $P^*_{x|y}\in\mathcal{P}_{x|y}$, then supervised learning and DSL has the same approximation error.
	\end{theorem}
	\begin{proof}
		By definition, we can verify both of the two approximation errors are zero.
	\end{proof}
	
	\section{Details about the Language Models for Marginal Distributions}\label{app:LM_details}
	We use the LSTM language models~\cite{sundermeyer2012lstm,mikolov2010recurrent} to characterize the marginal distribution of a sentence $x$, defined as $\prod_{i=1}^{T_x}P(x_i|x_{<i})$, where $x_i$ is the $i$-th word in $x$, $T_x$ denotes the number of words in $x$, and  the index $<i$ indicates $\{1,2,\cdots,i-1\}$. The embedding dimension and hidden node are both $1024$. We apply 0.5 dropout to the input embedding and the last hidden layer before softmax. The validation perplexities of the language models are shown in Table \ref{tab:LM_ppl_stats}, where the validation sets are the same.
	\begin{table}[!htpb]
		\centering
		\caption{Validation Perplexities of Language Models}
		\begin{tabular}{|c|c|c|c|c|c|}
			\hline
			\multicolumn{2}{|c|}{En$\leftrightarrow$Fr} & \multicolumn{2}{|c|}{En$\leftrightarrow$De} & \multicolumn{2}{|c|}{En$\leftrightarrow$Zh}\\
			\hline
			En & Fr & En & De & En & Zh\\
			\hline
			$88.72$ & $58.90$ & $101.44$ & $90.54$ & $70.11$ & $113.43$\\
			\hline
		\end{tabular}
		\label{tab:LM_ppl_stats}
	\end{table}
	
	For the marginal distributions for sentences of sentiment classification, we choose the LSTM language model again like those for machine translation applications. The two differences are: (i) the vocabulary size is $10000$; (ii) the word embedding dimension is $500$. The perplexity of this language model is $58.74$.	
{\small{
\bibliography{dualSLbib}

\begin{thebibliography}{33}
\providecommand{\natexlab}[1]{#1}
\providecommand{\url}[1]{\texttt{#1}}
\expandafter\ifx\csname urlstyle\endcsname\relax
  \providecommand{\doi}[1]{doi: #1}\else
  \providecommand{\doi}{doi: \begingroup \urlstyle{rm}\Url}\fi

\bibitem[Amodei et~al.(2016)Amodei, Anubhai, Battenberg, Case, Casper,
  Catanzaro, Chen, Chrzanowski, Coates, Diamos, et~al.]{amodei2016deep}
Amodei, Dario, Anubhai, Rishita, Battenberg, Eric, Case, Carl, Casper, Jared,
  Catanzaro, Bryan, Chen, Jingdong, Chrzanowski, Mike, Coates, Adam, Diamos,
  Greg, et~al.
\newblock Deep speech 2: End-to-end speech recognition in english and mandarin.
\newblock In \emph{33rd International Conference on Machine Learning}, 2016.

\bibitem[Bahdanau et~al.(2015)Bahdanau, Cho, and Bengio]{bahdanau2015neural}
Bahdanau, Dzmitry, Cho, Kyunghyun, and Bengio, Yoshua.
\newblock Neural machine translation by jointly learning to align and
  translate.
\newblock In \emph{International Conference on Learning Representations}, 2015.

\bibitem[Bartlett \& Mendelson(2002)Bartlett and
  Mendelson]{bartlett2002rademacher}
Bartlett, Peter~L and Mendelson, Shahar.
\newblock Rademacher and gaussian complexities: Risk bounds and structural
  results.
\newblock \emph{Journal of Machine Learning Research}, 3\penalty0
  (Nov):\penalty0 463--482, 2002.

\bibitem[Bengio et~al.(2003)Bengio, Ducharme, Vincent, and
  Jauvin]{bengio2003neural}
Bengio, Yoshua, Ducharme, R{\'e}jean, Vincent, Pascal, and Jauvin, Christian.
\newblock A neural probabilistic language model.
\newblock \emph{Journal of machine learning research}, 3\penalty0
  (Feb):\penalty0 1137--1155, 2003.

\bibitem[Boyd \& Vandenberghe(2004)Boyd and Vandenberghe]{boyd2004convex}
Boyd, Stephen and Vandenberghe, Lieven.
\newblock \emph{Convex optimization}.
\newblock Cambridge university press, 2004.

\bibitem[Dai \& Le(2015)Dai and Le]{dai2015semi}
Dai, Andrew~M and Le, Quoc~V.
\newblock Semi-supervised sequence learning.
\newblock In \emph{Advances in Neural Information Processing Systems}, pp.\
  3079--3087, 2015.

\bibitem[Graves et~al.(2013)Graves, Mohamed, and Hinton]{graves2013speech}
Graves, Alex, Mohamed, Abdel-rahman, and Hinton, Geoffrey.
\newblock Speech recognition with deep recurrent neural networks.
\newblock In \emph{Acoustics, speech and signal processing (icassp), 2013 ieee
  international conference on}, pp.\  6645--6649. IEEE, 2013.

\bibitem[He et~al.(2016{\natexlab{a}})He, Xia, Qin, Wang, Yu, Liu, and
  Ma]{he2016dual}
He, Di, Xia, Yingce, Qin, Tao, Wang, Liwei, Yu, Nenghai, Liu, Tie{-}Yan, and
  Ma, Wei-Ying.
\newblock Dual learning for machine translation.
\newblock In \emph{Advances In Neural Information Processing Systems}, pp.\
  820--828, 2016{\natexlab{a}}.

\bibitem[He et~al.(2016{\natexlab{b}})He, Zhang, Ren, and Sun]{he2015deep}
He, Kaiming, Zhang, Xiangyu, Ren, Shaoqing, and Sun, Jian.
\newblock Deep residual learning for image recognition.
\newblock In \emph{IEEE Conference on Computer Vision and Pattern Recognition},
  2016{\natexlab{b}}.

\bibitem[He et~al.(2016{\natexlab{c}})He, Zhang, Ren, and Sun]{he2016identity}
He, Kaiming, Zhang, Xiangyu, Ren, Shaoqing, and Sun, Jian.
\newblock Identity mappings in deep residual networks.
\newblock In \emph{European Conference on Computer Vision}, pp.\  630--645.
  Springer, 2016{\natexlab{c}}.

\bibitem[Hearst et~al.(1998)Hearst, Dumais, Osuna, Platt, and
  Scholkopf]{hearst1998support}
Hearst, Marti~A., Dumais, Susan~T, Osuna, Edgar, Platt, John, and Scholkopf,
  Bernhard.
\newblock Support vector machines.
\newblock \emph{IEEE Intelligent Systems and their Applications}, 13\penalty0
  (4):\penalty0 18--28, 1998.

\bibitem[IMDB(2011)]{IMDB}
IMDB.
\newblock Imdb dataset.
\newblock \emph{http://ai.stanford.edu/~amaas/data/sentiment/}, 2011.

\bibitem[Jean et~al.(2015)Jean, Cho, Memisevic, and Bengio]{cho2015using}
Jean, S{\'e}bastien, Cho, Kyunghyun, Memisevic, Roland, and Bengio, Yoshua.
\newblock On using very large target vocabulary for neural machine translation.
\newblock In \emph{ACL}, 2015.

\bibitem[Kingma \& Ba(2014)Kingma and Ba]{kingma2014adam}
Kingma, Diederik and Ba, Jimmy.
\newblock Adam: A method for stochastic optimization.
\newblock \emph{arXiv preprint arXiv:1412.6980}, 2014.

\bibitem[Krizhevsky \& Hinton(2009)Krizhevsky and
  Hinton]{krizhevsky2009learning}
Krizhevsky, Alex and Hinton, Geoffrey.
\newblock Learning multiple layers of features from tiny images.
\newblock 2009.

\bibitem[Maas et~al.(2011)Maas, Daly, Pham, Huang, Ng, and
  Potts]{maas2011learning}
Maas, Andrew~L, Daly, Raymond~E, Pham, Peter~T, Huang, Dan, Ng, Andrew~Y, and
  Potts, Christopher.
\newblock Learning word vectors for sentiment analysis.
\newblock In \emph{Proceedings of the 49th Annual Meeting of the Association
  for Computational Linguistics: Human Language Technologies-Volume 1}, pp.\
  142--150. Association for Computational Linguistics, 2011.

\bibitem[Mikolov et~al.(2010)Mikolov, Karafi{\'a}t, Burget, Cernock{\`y}, and
  Khudanpur]{mikolov2010recurrent}
Mikolov, Tomas, Karafi{\'a}t, Martin, Burget, Lukas, Cernock{\`y}, Jan, and
  Khudanpur, Sanjeev.
\newblock Recurrent neural network based language model.
\newblock In \emph{Interspeech}, volume~2, pp.\ ~3, 2010.

\bibitem[Mohri et~al.(2012)Mohri, Rostamizadeh, and
  Talwalkar]{mohri2012foundations}
Mohri, Mehryar, Rostamizadeh, Afshin, and Talwalkar, Ameet.
\newblock \emph{Foundations of machine learning}.
\newblock MIT press, 2012.

\bibitem[multi bleu(2015)]{multibleu}
multi bleu.
\newblock multi-bleu.pl.
\newblock
  \emph{https://github.com/moses-smt/\\mosesdecoder/blob/master/scripts/generic/multi-bleu.perl},
  2015.

\bibitem[Oord et~al.(2016)Oord, Dieleman, Zen, Simonyan, Vinyals, Graves,
  Kalchbrenner, Senior, and Kavukcuoglu]{oord2016wavenet}
Oord, Aaron van~den, Dieleman, Sander, Zen, Heiga, Simonyan, Karen, Vinyals,
  Oriol, Graves, Alex, Kalchbrenner, Nal, Senior, Andrew, and Kavukcuoglu,
  Koray.
\newblock Wavenet: A generative model for raw audio.
\newblock \emph{arXiv preprint arXiv:1609.03499}, 2016.

\bibitem[Papineni et~al.(2002)Papineni, Roukos, Ward, and
  Zhu]{papineni2002bleu}
Papineni, Kishore, Roukos, Salim, Ward, Todd, and Zhu, Wei-Jing.
\newblock Bleu: a method for automatic evaluation of machine translation.
\newblock In \emph{Proceedings of the 40th annual meeting on association for
  computational linguistics}, pp.\  311--318. Association for Computational
  Linguistics, 2002.

\bibitem[Pascanu et~al.(2013)Pascanu, Mikolov, and
  Bengio]{pascanu2013difficulty}
Pascanu, Razvan, Mikolov, Tomas, and Bengio, Yoshua.
\newblock On the difficulty of training recurrent neural networks.
\newblock \emph{ICML (3)}, 28:\penalty0 1310--1318, 2013.

\bibitem[Salimans et~al.(2017)Salimans, Karpathy, Chen, P.~Kingma, and
  Bulatov]{tim2017pixel}
Salimans, Tim, Karpathy, Andrej, Chen, Xi, P.~Kingma, Diederik, and Bulatov,
  Yaroslav.
\newblock Pixelcnn++: A pixelcnn implementation with discretized logistic
  mixture likelihood and other modifications.
\newblock In \emph{International Conference on Learning Representations}, 2017.

\bibitem[Shen et~al.(2016)Shen, Cheng, He, He, Wu, Sun, and
  Liu]{shen2016minimum}
Shen, Shiqi, Cheng, Yong, He, Zhongjun, He, Wei, Wu, Hua, Sun, Maosong, and
  Liu, Yang.
\newblock Minimum risk training for neural machine translation.
\newblock \emph{ACL}, 2016.

\bibitem[Sundermeyer et~al.(2012)Sundermeyer, Schl{\"u}ter, and
  Ney]{sundermeyer2012lstm}
Sundermeyer, Martin, Schl{\"u}ter, Ralf, and Ney, Hermann.
\newblock Lstm neural networks for language modeling.
\newblock In \emph{Interspeech}, pp.\  194--197, 2012.

\bibitem[Theis et~al.(2015)Theis, Oord, and Bethge]{theis2015note}
Theis, Lucas, Oord, A{\"a}ron van~den, and Bethge, Matthias.
\newblock A note on the evaluation of generative models.
\newblock \emph{arXiv preprint arXiv:1511.01844}, 2015.

\bibitem[Tibshirani(1996)]{tibshirani1996regression}
Tibshirani, Robert.
\newblock Regression shrinkage and selection via the lasso.
\newblock \emph{Journal of the Royal Statistical Society. Series B
  (Methodological)}, pp.\  267--288, 1996.

\bibitem[van~den Oord et~al.(2016{\natexlab{a}})van~den Oord, Kalchbrenner,
  Espeholt, Vinyals, Graves, et~al.]{van2016conditional}
van~den Oord, Aaron, Kalchbrenner, Nal, Espeholt, Lasse, Vinyals, Oriol,
  Graves, Alex, et~al.
\newblock Conditional image generation with pixelcnn decoders.
\newblock In \emph{Advances in Neural Information Processing Systems}, pp.\
  4790--4798, 2016{\natexlab{a}}.

\bibitem[van~den Oord et~al.(2016{\natexlab{b}})van~den Oord, Kalchbrenner, and
  Kavukcuoglu]{van2016pixel}
van~den Oord, Aaron, Kalchbrenner, Nal, and Kavukcuoglu, Koray.
\newblock Pixel recurrent neural networks.
\newblock In \emph{33rd International Conference on Machine Learning},
  2016{\natexlab{b}}.

\bibitem[WMT(2014)]{wmt14}
WMT.
\newblock Wmt dataset for machine translation.
\newblock \emph{{http://www.statmt.org/wmt14/translation-task.html}}, 2014.

\bibitem[Wu et~al.(2016)Wu, Schuster, Chen, Le, Norouzi, Macherey, Krikun, Cao,
  Gao, Macherey, et~al.]{wu2016google}
Wu, Yonghui, Schuster, Mike, Chen, Zhifeng, Le, Quoc~V, Norouzi, Mohammad,
  Macherey, Wolfgang, Krikun, Maxim, Cao, Yuan, Gao, Qin, Macherey, Klaus,
  et~al.
\newblock Google's neural machine translation system: Bridging the gap between
  human and machine translation.
\newblock \emph{arXiv preprint arXiv:1609.08144}, 2016.

\bibitem[Xia et~al.(2017)Xia, Bian, Qin, Yu, and Liu]{DI2017xia}
Xia, Yingce, Bian, Jiang, Qin, Tao, Yu, Nenghai, and Liu, Tie{-}Yan.
\newblock Dual inference for machine learning.
\newblock In \emph{The 26th International Joint Conference on Artificial
  Intelligence}, 2017.

\bibitem[Zeiler(2012)]{zeiler2012adadelta}
Zeiler, Matthew~D.
\newblock Adadelta: an adaptive learning rate method.
\newblock \emph{arXiv preprint arXiv:1212.5701}, 2012.

\end{thebibliography}
\bibliographystyle{icml2017}	
}}

\end{document}